\documentclass{article}

\usepackage[utf8]{inputenc} 
\usepackage[T1]{fontenc}    
\usepackage{hyperref}       
\usepackage{url}            
\usepackage{booktabs}       
\usepackage{amsfonts}       
\usepackage{nicefrac}       
\usepackage{microtype}      
\usepackage{mathtools,amssymb}
\usepackage{microtype}
\usepackage{graphicx}
\usepackage{subfigure}
\usepackage{booktabs} 

\usepackage[T1]{fontenc}
\usepackage{lmodern}

\usepackage{bbm}
\usepackage{amsmath}
\usepackage{amsthm}

\newtheorem{thm}{Theorem}
\newtheorem{defn}{Definition}
\newtheorem{lemma}{Lemma}

\newcommand{\norm}[1]{\left\lVert#1\right\rVert}

\newcommand{\nf}{\mathcal{F}}

\newcommand{\nr}{\mathbb{R}}

\newcommand{\np}{\mathbb{P}}

\newcommand{\nE}{\mathbb{E}}

\newcommand{\Strategy}{\mathcal{S}}
\newcommand{\strategy}{S}
\newcommand{\cX}{\mathcal{X}} 
\newcommand{\cY}{\mathcal{Y}} 
\newcommand{\obsS}{\mathcal{X}} 
\newcommand{\choS}{\mathcal{Y}} 
\newcommand{\obs}{X} 
 
\newcommand{\cho}{y} 
\newcommand{\obsF}{X} 
\newcommand{\ml}{u} 
\newcommand{\optimal}{\mathcal{V}^*} 

\newcommand{\shortcite}{\cite}

\newcommand{\eqdef}{\triangleq}

\title{Nonparametric Online Learning Using Lipschitz Regularized Deep Neural Networks}

\author{
	Guy Uziel \\
	Department of Computer Science\\
	Technion - Israel Institute of Technology\\
}

\begin{document}

\maketitle
\begin{abstract} 
 Deep neural networks are considered to be state of the art models in many offline machine learning tasks. However, their performance and generalization abilities in online learning tasks are much less understood.  Therefore,  we focus on online learning and tackle the challenging problem where the underlying process is stationary and ergodic and thus removing the i.i.d. assumption and allowing observations to depend on each other arbitrarily.  
 
 We prove the generalization abilities of Lipschitz regularized deep neural networks and show that by using those networks, a convergence to the best possible prediction strategy is guaranteed.  
 	

\end{abstract} 

\section{Introduction }

In recent years, deep neural networks have been applied to many off-line machine learning tasks. Despite their state-of-of-the-art performance, the theory behind their generalization abilities is still not complete. When turning to the online domain even much less is known and understood both from the practical use and the theoretical side. Thus, the main focus of this paper is exploring the theoretical guarantees of deep neural networks in online learning under general stochastic processes.

In the traditional online learning setting, and in particular in sequential prediction under uncertainty, the learner is evaluated by a  loss function that is not entirely known at each iteration \cite{CesaL2006}. In this work, we study online prediction focusing on the challenging case where the unknown underlying process is stationary and ergodic, thus allowing observations to depend on each other arbitrarily.

Many papers before have considered online learning under stationary and ergodic sources and in various application domains. For example, in online portfolio selection, \cite{Gyorfi2007,GyorfiLU2006,GyorfiS2003,Uziel2018,LiHG2011} proposed nonparametric online strategies that guarantee, under mild conditions, convergence to the best possible outcome. \cite{BiauKLG2010,BiauP2011} has considered the setting of time-series prediction. Another mentionable line of research is the works of \cite{GyorfiL2005} regarding the online binary classification problem under such processes.

A common theme to all of these algorithms is that the asymptotically optimal strategies are constructed by combining the predictions of simple experts. The experts are constructed using a  pre-defined  nonparametric density estimation method \cite{nadaraya1964estimating,rosenblatt1956remarks}. The algorithm, using the experts, implicitly learns the best possible online strategy, which is, following a well-known result of \cite{algoet1994}, a Borel-measurable function.

On the other hand, neural networks are universal approximation functions \cite{hornik1989}, and thus, it seems natural to design an online learning strategy that will be able to learn the best Borel-measurable strategy explicitly.

The theoretical properties of neural networks have been studied in the off-line setting, where generalization bounds have been obtained via VC dimension analysis of neural networks \cite{Bartlett1997}. However, the generalization rates obtained in the above paper are applicable only for low-complexity networks.  Another approach has investigated the connection between generalization and stability \cite{Bousquet2002,xu2012}. \cite{bartlett2017} has suggested the Lipschitz constant of the network as a candidate measure for the Rademacher complexity. \cite{cranko2018lipschitz} showed that Lipschitz regularization could be viewed as a particular case of distributional robustness and \cite{Oberman2018lipschitz} derived generalization bounds for Lipschitz regularized networks.

Recently, there has been a growing interest in the properties and behavior of neural networks in the online setting, and there have been attempts to apply than in various domains. For example, in online portfolio selection \cite{jiang2017deep,wang2018lstm}, in electricity forecasting \cite{kuo2018electricity,mujeeb2019deep} in time series prediction \cite{zhu2017deep,chen2018short}.

Still, applying deep training models to the online regime is a hard problem and training approaches that were developed for the batch setting, are not seem to be suitable for the online setting due to lack of ability to tune the model in an online fashion. Therefore, \cite{Sahoo2018} have suggested a way to integrate the tuning of the model with online learning. A different approach was proposed recently by \cite{borovykh2019generalisation}, who have studied fully-connected neural networks in the context of time series forecasting.

The latter has developed a way to train a neural network and detect whether the training was successful in terms of generalization abilities. Not surprisingly, one of the conclusions of the previous work was that for the learned neural function to perform well on the test data as well one has to regularized the neural network. Thus, manifesting the understanding that regularization is a vital component for ensuring that a  better generalization result from the learned function and for controlling the trade-off between the complexity of the function and its ability to fit the data.

Unlike recent papers, we are not focusing here on the training procedure and on how to ensure proper training of a given network in the online setting. The focus of this paper is dealing with the theoretical guarantees of using a deep learning strategy as an online forecasting strategy. More specifically, we study the guarantees of using Lipschitz regularized deep neural networks and their abilities to asymptotically converge for the best possible outcome when the underlying process is stationary and ergodic.


As far as we know and despite all the developing theory of generalization of neural network in the offline setting, we are not aware of papers dealing with the properties of neural network in the sequential prediction setting when the underlying process is non-i.i.d.


The paper is organized as follows: In Section~\ref{sec:formulation}, we define the nonparametric sequential prediction framework under a jointly stationary and ergodic process and we present the goal of the learner. In Section~\ref{sec:Algorithm}, we present  and provide proofs for their guarantees.

\section{Nonparmetric Online Learning}
\label{sec:formulation}

We consider the following prediction game.
Let $\cX \eqdef [0,1]^n \subset \nr^n$ be a compact \emph{ observation space}.\footnote{For convenience, we focus here on the space of $[0,1]^n$, however, the results applied for any compact observation space.} 
At each round, $t = 1, 2, \ldots$, the player is required to make a prediction $y_t\in \choS $, where $\choS \subset \nr ^m$ is a compact and convex set,  based on past observations,
$X^{t-1}_1 \eqdef (x_1, \ldots , x_{t-1})$ and, $x_i \in \cX$ ($X_1^0$ is the empty observation).   After making the prediction $y_t$, the observation $x_t$ is revealed and the player suffers   loss, $\ml(y_t,x_t)$, where $\ml$ is a  real-valued Lipschitz continuous function and  convex w.r.t. her first argument. 
We remind below the definition of a Lipschitz function:
\begin{defn}
    Let $X,Y$ be two normed spaces, a function $f:X \rightarrow Y $ will be called $L$-Lipschitz function if:
    \begin{gather*}
        L = \sup_{x_1,x_2 \in X} \frac{\norm{f(x_1)-f(x_2)}_Y}{\norm{x_1-x_2}_X} 
    \end{gather*}
We denote $\mathcal{L}_L$ to be the class of L-Lipschitz functions, and by $\mathcal{L}_\infty$ the class of all Lipschitz functions.
\end{defn}
Without loss of generality we assume that $u$ is $1-$Lipschitz.

We view  the player's prediction strategy as a sequence 
$\Strategy \eqdef \{\strategy_t\}^\infty_{t=1}$ of forecasting functions 
$\strategy_t : \cX^{(t-1)} \rightarrow \cY$;  that is,
the player's prediction  at round $t$ is given by $\strategy_t(X^{t-1}_1 )$ (for brevity, we denote $\strategy(X^{t-1}_1 )$). 

Throughout the paper we assume that $x_1, x_2, \ldots$ are realizations of random variables $X_1, X_2, \ldots$ such that the stochastic process $(X_t)^\infty_{-\infty}$\footnote{By Kolmogorov's extension theorem, the stationary and ergodic process $(X_t)^\infty_{1}$ can be extended to $(X_t)^\infty_{-\infty}$ such that the ergodicity holds for both $t\rightarrow \infty$ and $t\rightarrow -\infty$  (see, e.g., Breiman~\shortcite{Breiman1992})}  is jointly stationary and ergodic with full support on  $\obsS$.   
The player's goal is to play the game with a strategy that minimizes the average $\ml$-loss,  $$\frac{1}{T}\sum_{t=1}^{T}\ml(\strategy(X_1^{t-1}),x_t).$$

The well known result of Algoet~\cite{algoet1994}, formulated below, states a lower bound for the performance of any  online strategy (without the power of hindsight).
\begin{thm}[\cite{algoet1994}]
Let $\Strategy$ be any prediction strategy, then the following holds a.s.
\begin{gather*}
 \liminf_{T\rightarrow \infty} \frac{1}{T}\sum_{t=1}^{T}\ml(\strategy(X_1^{t-1}),x_t) \geq
 \nE\left[ \max_{\cho \in \choS()} \nE_{\np_\infty}\left[\ml(\cho,\obs_0)\right]\right] 
\end{gather*}
where $\np_\infty$ is the regular conditional probability distribution of $\obsF_0$  given $\nf_\infty$ (the $\sigma$-algebra generated by the  past observations $\obs_{-1},\obs_{-2},\ldots$) and the maximization is over the $\nf_\infty$-measurable functions.
\end{thm}

  Therefore we define the optimal quantity as follows:
\begin{gather*}
 \optimal \eqdef 
\nE\left[ \max_{\cho \in \choS()} \nE_{\np_\infty}\left[\ml(\cho,\obs_0)\right]\right] 
\end{gather*}   

The above result defines the holy-grail in sequential prediction under stationary and ergodic processes, and the vast majority of papers dealing with this setting have proposed methods for pursuing this asymptotic boundary. Moreover, according to \cite{algoet1994} one approach to pursue $\optimal$ is by creating estimates that converge (weakly)  to the conditional distribution $\nE_{\np_\infty}$, and thus implicitly estimating the best Borel-measurable strategy. 

This approach has been applied in many  papers before \cite{Gyorfi2007,GyorfiLU2006,GyorfiS2003,Uziel2018,LiHG2011,BiauKLG2010,BiauP2011}
by carefully aggregating a countable (infinite) array of experts each accompanied with the chosen density estimation method. 

The guarantees provided in the above papers are always asymptotic. It is no coincidence, as it is well-known that without any regularity conditions on the underlying process one cannot achieve high probability bound and the results are always asymptotic since the martingale convergence theorem does not provide any guaranteed rate \cite{DevroyeGG2013}. 

In practice approximate implementations of those strategies (applying only a finite set of experts), however, turn out to work exceptionally well and, despite the inevitable approximation, are reported \cite{GyorfiUW2008,GyorfiS2003,GyorfiLU2006} to significantly outperform strategies designed to work in an adversarial, no-regret setting, in various domains.
 
Therefore, just like previous works \cite{Meir2000,Modha1998} dealing with high probability bounds for this setting, we impose regularity conditions and focus our attention on bounded memory processes with mixing properties.

The first regularity condition that we make is that the memory of the process is bounded by a known bound $d$, thus, ensuring that the optimal prediction strategy do not depend the infinite tail but only of past $d$ observations. More formally, we get the following:
\begin{gather*}
 \optimal \eqdef 
\nE\left[ \max_{\cho \in \choS()} \nE_{\np_\infty}\left[\ml(\cho,\obs_0)\right]\right] = \nE\left[ \max_{\cho \in \choS()} \nE_{\np_d}\left[\ml(\cho,\obs_0)\right]\right]
\end{gather*}

where $\np_d$ is the distribution conditioned on the last $d$ observations.

The second common  regularity condition is regarding the mixing properties of the process. The meaning of the mixing property is that the process depends weakly on its past. Below we remind the definition of a $\beta$-mixing process\footnote{Other  definitions  of  mixing  exists (e.g., see  \cite{doukhan1991mixing}).}: 

\begin{defn}[$\beta$-mixing process]
	Let $\sigma_l = \sigma(\obsF_1^l)$
	and $\sigma_{l+m} = \sigma(\obsF_{l+m}^\infty)$
	be the sigma-algebras of events generated
	by the random variables $\obsF_{l}^l = (\obs_{1}, \obs_{2},\ldots ,\obs_{l})$ and $
	\obsF_{l+m}^\infty = (\obs_{l+m}, \obs_{l+m+1}, . . .)$,
	respectively. The coefficient of absolute regularity, $\beta_m$, is given by 
	\begin{gather*}
	\beta_m = \sup_{l\geq 1} \nE \sup_{B\in\sigma_{l+m} } |\np(B\mid \sigma_{l+m}) - \np(B)| 
	\end{gather*}
	where the expectation is taken with respect to $\sigma_l$. A stochastic process is said to be absolutely
	regular, or $\beta$-mixing, if $\beta_m\rightarrow 0$ as $m \rightarrow \infty$.
\end{defn}

If also $\beta_m\leq cm^{-r}$ for some positive constants $c,r$   then we say that  the process has \emph{algebraic} mixing rate, and \emph{exponential} mixing rate if $\beta_m\leq c\exp{(m^{-r})}$.

Using the above definition we focus  our attention on $\beta$-mixing processes.

It is important again to emphasize the importance of those assumptions for achieving the finite regret bounds and efficient algorithms.  The first condition is essential in order to deal with finite memory predictors and the second assumption (or an equivalent assumption on the mixing properties of the process) is needed in order to establish high probability bounds on the finite sample.


%

 Summarizing the above, we can state the player's goal using  the following definitions:
 \begin{defn}
      Strategy $\strategy$ will be called  \emph{$\beta$-universal} if  for a $\beta$-mixing process the following holds a.s.:
          \begin{gather*}
	\limsup_{T\rightarrow \infty}   \frac{1}{T}\sum_{t=1}^{T}\ml(\strategy(X_{t-d}^{t-1}),x_t) =\optimal  
\end{gather*}
 \end{defn}
\section{$\beta$-universal strategy}

We now turn to describe our approach to learn the best Borel-strategy explicitly. This approach does not involve the use of any density estimation techniques as done in previous papers, but rather by the use of Lipschitz regularized deep neural networks. The work with Lipschitz functions is, of course, more appealing than dealing with the abstract space of Borel-measurable functions. As we later show in our proof,  it is indeed enough to learn the best Borel-function by learning the best Lipschitz function, since the Lipschitz functions are dense in the space of Borel functions. In other words, any Borel-function can be approximated arbitrarily by a Lipschitz function. However, using this observation, we can deduce that in order to achieve the universality we will have to increase the complexity of the network as time passes (or equivalently reduce the regularization of the network).

We now proceed to describe two of the main components in our method.

\subsection{Lipschitz regularization in neural networks}

Much attention has been given to regularization methods for neural networks, which are necessary for achieving good generalization results. Regularization includes several methods such as the well known $l_1$ or $l_2$ regularization and dropout \cite{srivastava2014dropout,han2015learning}.
 A suitable selection of the batch size and the learning rate have also been shown to prevent low generalization abilities has been shown in previous work, e.g., \cite{keskar2016large,zeiler2013stochastic}.

Lipschitz regularization methods were discussed in several papers before \cite{gouk2018regularisation,miyato2018spectral,finlay2018improved}. Several methods for computing and regularizing the networks have been proposed.  The Lipschitz constant of a given network, $L_{model}$ can be calculated using simple solvers; however,  it is computationally intensive \cite{katz2017reluplex}. A more  feasible approach proposed by \cite{szegedy2013intriguing} is by looking at a given neural network $f$ as a composition of functions $f(\obs) = (\phi_1 \circ \phi_2 \circ \ldots \phi_k)(\obs)$. Thus, using the fact that a composition of a $L_1$-Lipschitz function with $L_2$-Lipschitz function, results in a  $L_1L_2$-Lipschitz function. Therefore, we can derive the following  upper bound:

\begin{gather}
\label{ex:reg}
    L_{model} \leq \Pi_{i=1}^k L_{\phi_i}
\end{gather}
where $L_{\phi_i}$ is the Lipschitz constant of $\phi_i$.

Implementation methods based on Equation~\eqref{ex:reg} have been presented recently in \cite{gouk2018regularisation,miyato2018spectral}. This upper bound, as shown in \cite{finlay2018improved},  is not tight because it does not take into account the zero coefficients in weight matrices caused by the activation functions. Unique architectures can also be used to implement  Lipschitz regularization. For example, on a restricted architecture renormalized the weight matrices of each layer to posses unit norm \cite{liao2018surprising}.

We now turn to describe the block independent method of \cite{yu1994rates}. This method allows us to transform the underlying mixing process into an independent one. Using this transformation, we will be able to use existing theory and methods for deriving uniform bound results for i.i.d. empirical process.

\subsection{Block independent method}

In her seminal paper \cite{yu1994rates}, Yu derived uniform laws of large numbers for mixing processes. The primary tool for deriving the result was using the  block independent method. This method involves the construction of an independent block sequence which was shown to be close to the original process in a (well-defined) probabilistic sense. 

The construction goes as follows: first, we divide the original sequence $\obsF_1^T$ into $2\mu_T$ blocks, each of size $a_T$.\footnote{We will not  be concerned with the remainder terms, which become insignificant as the sample size increases.} The blocks are then numbered and grouped in the following way:
\begin{gather*}
H_j=\{i: 2(j-1)a_T+1 \leq i \leq (2j-1)a_T  \} \\
 T_j=\{i: (2j-1)a_T+1 \leq i \leq 2ja_T \}.
\end{gather*}
We denote the random variables corresponding to the blocks $H_j$ and $T_j$ respectively by
\begin{gather*}
\obs_{(H_j)} \eqdef \{\obs_i, i\in H_j  \} \quad \quad \quad \quad  \obs_{(T_j)} \eqdef \{\obs_i, i\in T_j  \}
\end{gather*}
 
and the  sequence of $H$-blocks  by $\obs_{a_T}\eqdef \{\obs_{(H_j)}\}_{j=1}^{\mu_T}$.


 Now, we can construct a sequence of independently distributed blocks, $\Psi_{j}=\{ \psi_i : i\in  H_j  \}$, such that the sequence is independent of $\obsF_1^T$, and each block has the same distribution as the block $\obs_{(H_j)}$ from the original sequence. We denote this sequence by $\Psi^{a_T}=\{\Psi_{j}\}_{j=1}^{\mu_T}$. Because the original process is  stationary, the blocks constructed are not only independent but also identically distributed.  Moreover, by appropriately selecting the number of blocks, $\mu_T$, depending on the mixing nature of the sequence, we can relate properties of the original sequence $\obsF_1^T$, to those of the independent block sequence. 

We now show how we can translate a uniform bound written  in terms of  the original process to be written in terms of the constructed block independent process. Given a  bounded function class $\mathcal{F}$,

\begin{gather*}
    \np\left( \sup_{f\in \mathcal{F}} \left| \frac{1}{T}\sum_{t=1}^T f(\obs_t) - \nE f(\obs) \right| >\epsilon \right) =  \nonumber\\ 
     \np\left( \sup_{f\in \mathcal{F}} \left| \frac{1}{T}\sum_{j=1}^{\mu_T} \sum_{i\in H_j} f(\obs_i) - \frac{1}{2a_T}\nE \sum_{i\in H_j}f(\obs_i) + \frac{1}{T}\sum_{j=1}^{\mu_T} \sum_{i\in T_j} f(\obs_i)  -\frac{1}{2a_T}\nE \sum_{i\in T_j}f(\obs_i)    \right| >\epsilon \right) \nonumber \\ \leq \np\left( \sup_{f\in \mathcal{F}} \left| \frac{1}{T}\sum_{j=1}^{\mu_T} \sum_{i\in H_j} f(\obs_i) - \frac{1}{2a_T}\nE \sum_{i\in H_j}f(\obs_i)\right| >\epsilon \right)  \nonumber\\
     +  \np\left( \sup_{f\in \mathcal{F}} \left| \frac{1}{T}\sum_{j=1}^{\mu_T} \sum_{i\in T_j} f(\obs_i)  -\frac{1}{2a_T}\nE \sum_{i\in T_j}f(\obs_i)    \right| >\epsilon \right).
\end{gather*}

The above inequality was derived by decomposing the original sequence into the different blocks.

 Using the notation $f_{H_j}(X) \eqdef \sum_{i\in H_j} f(\obs_i)$  and $f_{T_j}(X) \eqdef \sum_{i\in T_j} f(\obs_i)$ and using $T=2\mu_Ta_T$ we can conclude from above the following
 \begin{gather}
    \np\left( \sup_{f\in \mathcal{F}} \left| \frac{1}{T}\sum_{t=1}^T f(\obs_t) - \nE f(\obs) \right| >\epsilon \right) \leq 2\np\left( \sup_{f\in \mathcal{F}} \left| \frac{1}{T}\sum_{j=1}^{\mu_T} f_{H_j}(X) - \nE f_{H}\right| > a_T\epsilon\right)
\end{gather}

We have expressed the sequence as the block dependent way, it remains, of course, to relate the properties of the original sequence to the block independent one; therefore we state the following lemma by \cite{yu1994rates}. The proof is based mainly on mixing properties.


\begin{lemma}
    \label{lem:yu}
	Suppose $\mathcal{F}:\nr^m \rightarrow \nr$ is a permissible class\footnote{ Permissible in the sense specified in \cite{pollard2012convergence}.} of bounded functions, then
	\begin{gather*}
	\np\left( \sup_{f \in \mathcal{F}} \left|\frac{1}{T}\sum_{i=1}^T f(\obs_i)-\mathbb{E}f(\obs) \right| >\epsilon \right) \leq  \\ 2 \hat{\np}\left( \sup_{f \in \mathcal{F}} \left|\frac{1}{\mu_T}\sum_{i=1}^{\mu_T} f_{H_j}(\Psi_{j})-\hat{\mathbb{E}}f(\Psi) \right| > a_T\epsilon \right) + 2\mu_T\beta_{a_T} 
	\end{gather*}
where $\hat{\np}$,$\hat{\mathbb{E}}$ relates to the block independent process.
\end{lemma}

%



\section{$\beta$-Universal Strategy}
\label{sec:Algorithm}

In this section, we present and show that our procedure is guaranteed to achieve $\beta$-universality.
Our training procedure goes as follows: first relying on the first regularity condition (the process is memory bounded)  the network gets as input a context window of length $d$ containing $d$ consecutive examples. The goal of the neural network, which is regularized by a Lipschitz regularized, $L$ is to fit the training examples with the best $L$-Lipschitz function. In practice, it might be hard to get an indication that the neural network has indeed fully minimized the empirical loss; however, in cases where $\optimal = 0$ it might be reasonable, relying on the common assumption that neural networks may get zero training error, an assumption that was justified in \cite{zhang2016understanding}. Thus, the only thing one should care of is that the neural network has enough parameters  so that the empirical data can be fitted. 

Therefore, we will assume that $f_1,f_2,\ldots$ are neural networks with $L_1,L_2,\ldots$ as the corresponding Lipschitz-constants, each fully minimizing the empirical loss $\frac{1}{T}\sum_{t=1}^T \ml(f_T(\obs_{t-d}^{t-1}),\obs_t)$.

Regardless of ensuring that the network can fit the training data, using the fact that we are estimating the best Borel-measurable strategy, we will have to increase the Lipschitz constant as time progresses.


As in many generalization problems, we need to consider uniform-bound over function classes. The reason for that is that the learner is optimizing his model on a realization of the process, and thus, in order the guarantee the generalization we need to provide a probability bound for the event that the distance between the empirical loss minimizer to the minimizer of the expected loss. This, as formulated in the famous lemma below  (see, e.g., \cite{DevroyeGG2013}), can be achieved by a uniform bound over the class of functions.

\begin{lemma}
    \label{lem:gyo}
	Let $\mathcal{F}$ be a class of functions and let $f\in\mathcal{F}$ be a minimizer of the empirical loss function, then the following holds:
	\begin{gather*}
	\frac{1}{T} \sum_{t=1}^{T}\ml(f(X_{t-d}^{t-1}),\obs_t) - \inf_{g\in \mathcal{F}}\nE \ml(g(\obs),\obs) \leq 
	2\sup_{ g\in \mathcal{F}} \left|\frac{1}{T} \sum_{t=1}^{T}\ml(g(X_{t-d}^{t-1}),\obs_t)-   \nE \ml(g(\obs),\obs)   \right|
	\end{gather*} 
\end{lemma}

We will now state a concentration bound that will be used during the proof of our main theorem. This lemma is a uniform tail bound on the deviation of a Lipschitz function, this result is also known as the discrepancy result \cite{talagrand2006generic}.


\begin{lemma}
\label{lem:con}
Let $\obs_1,\ldots,\obs_n$ be i.i.d. sampled from a distribution $\np$ whose support is $[0,1]^m$ and let $\mathcal{L}_L$ be the class of functions with  Lipschitz constant of $L$. Then for $\epsilon > 0$,
\begin{gather*}
\np \left( \sup_{f\in \mathcal{L}_L} \left| \frac{1}{T}\sum_{i=1}^T f(x_i) -  \nE f(\obs) \right| > \epsilon \right) \leq  
2D^\frac{-m}{m+2}\exp\left(-\log(T)\left[C_1D-1 \right] \right)
\end{gather*}
where $D=\frac{ T}{\log(T)}\left(\frac{ \epsilon}{C_2L} \right)^{m+2}$, and $C_1,C_2$ are positive constants.
\end{lemma}

We now turn to prove our main result, stating that under suitable conditions the strategy of using a Lipschitz regularized deep neural network which fully minimizes the empirical loss, will a.s. achieve the same asymptotic results as using the best possible strategy.

\begin{thm}
Let $(X_t)^\infty_{-\infty}$ be an exponential $\beta$-mixing process. Suppose that $f_1,f_2\ldots$ are Lipschitz functions with $L_1,L_2,\ldots$ Lipschitz constants respectively and that each $f_t$   minimizes  the empirical loss up to time $t$. Then, the following holds a.s.:
\begin{gather*}
    \lim_{T\rightarrow \infty}\sum_{t=1}^{T}\ml(f_T(X_{t-d}^{t-1}),\obs_t) = \optimal
\end{gather*}
 
\end{thm}
\begin{proof}
Fixing time instance $T$, we first denote for $\epsilon >0$ the following event

\begin{gather*}
    A^\epsilon_T = \Bigg\{ \left| \frac{1}{T} \sum_{t=1}^{T}\ml(f_T(X_{t-d}^{t-1}),\obs_t) -\inf_{g\in \mathcal{L}_{L_T}}\nE \ml(g(\obs),\obs)  \right| > \epsilon\Bigg\}.
\end{gather*}

Since $f_T \in \mathcal{L}_{L_T}$  fully minimizes the empirical loss, we get by using Lemma~\ref{lem:gyo} that
\begin{gather}
\np(A_T^\epsilon) = \np \left( \left| \frac{1}{T} \sum_{t=1}^{T}\ml(f_T(X_{t-d}^{t-1}),\obs_t) -\inf_{g\in \mathcal{L}_{L_T}}\nE \ml(g(\obs),\obs)  \right| > \epsilon \right) \nonumber \\ \leq \np \left( \sup_{g\in \mathcal{L}_{L_T}} \left| \frac{1}{T} \sum_{t=1}^{T}\ml(g(X_{t-d}^{t-1}),\obs_t)-   \nE \ml(g(\obs),\obs) \right| > \frac{\epsilon}{2} \right) \label{eq:3}
\end{gather}
Now, since our process is a $\beta$-mixing process we can use the block independent method. By using  Lemma~\ref{lem:yu} we get that
\begin{gather}
\eqref{eq:3}  \leq 2\hat{\np} \left( \sup_{g\in \mathcal{L}_{L_t}} \left| \frac{1}{\mu_T} \sum_{t=1}^{\mu_T}\ml^g_{\Psi_t}-   \hat{\nE} \ml^g_{\Psi}\right| > \frac{a_T\epsilon}{2} \right) + 2\mu_T \beta_{a_T-d}, \label{eq:4}
\end{gather}
where, $\ml^g_{\Psi_j} = \sum_{\psi_i\in \Psi_j} \ml(g(\psi_{i-d}^{i-1}),\psi_i)$.

Since $g$ is a $L_T$-Lipschitz and $u$ is  $1$-Lipschitz, we get that the Lipschitz constant of $\ml^g_{\Psi_t}$ is $a_TL_T$. By applying Lemma~\ref{lem:con} we can deduce  that
\begin{gather}
\eqref{eq:4}  \leq 2D^\frac{-m}{m+2}\exp\left(-\log(\mu_T)\left[C_1D-1 \right] \right)  +  2\mu_T \beta_{a_T-d} \label{eq:5}
\end{gather}
where 
\begin{gather*}
D=\frac{ \mu_T}{\log(\mu_T)}\left(\frac{\epsilon}{C_2L_t} \right)^{m+2}
\end{gather*}
Summarizing the above, we have bounded the probability for the deviation  of the empirical performance of the best empirical $L_t$-Lipschitz function from its expected performance.

Now, if, for example, we set $\mu_T = \mathcal{O}(\sqrt{T})$ (and thus $a_T= \mathcal{O}(\sqrt{T})$) and set $L_t= \mathcal{O}((\log{T})^{1/m+2})$,
we get, due to the exponential mixing properties of the process that
\begin{gather*}
\np \left( A_T^\epsilon \right) \leq \delta_T
\end{gather*}
where $\sum_{t=1}^\infty \delta_T <\infty $.

Therefore, using the Borel-Cantelli lemma, we get that for every $\epsilon>0$ the events $A_T^\epsilon$ occur finite number of times.  Thus, we can conclude that 
\begin{gather*}
	\lim_{T\rightarrow \infty}   \frac{1}{T}\sum_{t=1}^{T}\ml(f_T(X_{t-d}^{t-1}),X_t) =  \inf_{f\in \mathcal{L}_\infty}\nE(\ml(f(\obs),\obs)) 
\end{gather*}

To conclude the proof, we need to show that indeed 

\begin{gather*}
 \inf_{f\in \mathcal{L}_\infty}\nE(\ml(f(\obs),\obs)) = \optimal  
\end{gather*}

 This can be achieved by using two well known theorems, first since we are dealing with a compact metric space we can use the  Stone-Weierstrass theorem and the dominated convergence theorem \cite{Stout1974} to conclude that the space of Lipschitz functions is dense in the space of continuous function. Thus, $$ \inf_{f\in \mathcal{L}_\infty}\nE(\ml(f(\obs),\obs)) = \inf_{ g\in \mathcal{C}}\nE(\ml(f(\obs),\obs)), $$ where $\mathcal{C}$ is the space of real continuous functions.
  
  Now  only left to show that every Borel-function can be approximated by a continuous one, this can be achieved using Lusin's theorem which states that for a given Borel-function $f$ there exists a continuous function $g$ such that $\np(|f-g|>\delta)<\epsilon$. Therefore we get the following 
\begin{gather*}
     \left|\nE  \ml(f(x),x)- \nE\ml(g(x),x) \right| \leq  \nE \left|  \ml(f(x),x)- \nE\ml(g(x),x) \right|\\ \leq 
    \nE \left| \ml(f(x),x)-\ml(g(x),x) \right|I_{|f-g|>\delta} + \nE \left| \ml(f(x),x)-\ml(g(x),x) \right|I_{|f-g|<\delta} \\
    \leq D\epsilon + M\delta
\end{gather*}
and since the bound can be arbitrarily small we get that
\begin{gather*}
 \optimal = \inf_{ g\in \mathcal{C}}\nE(\ml(f(\obs),\obs))   
\end{gather*}

and therefore our proof is finished.
\end{proof}
The results below are stated only for exponential mixing rate processes. However, with a careful choice of $L_t,M_t$ the results apply also for algebraic rate processes.
\section{Conclusions}

Deep online learning is a challenging task both from practical and theoretical sides due to the streaming nature of the data.  The results presented in this paper are an attempt to better understand the abilities of neural networks in the context of online prediction focusing on the challenging case when the underlying process is stationary and ergodic. 

We presented a procedure using Lipschitz regularized deep neural network and proved that under $\beta$-mixing process with algebraic rate the procedure is guaranteed to be weak $\beta$-universal and strong $\beta$-universal if the process is with exponential rate.  

There are still many interesting challenging research directions in the deep online learning regime, both theoretically and practically:
From the theoretical side, whether it is possible to generalize our results for weaker notions of mixing and possibly the ability of neural networks to converge to $\optimal$ in the more general case of stationary and ergodic. This challenge becomes even harder when non-stationary processes are considered. From the practical side, as we saw in our paper, there is a trade-off between the ability of the network to generalize and the need to add more complexity to the model. Therefore the development of robust training procedures and measures to ensure the appropriate training of such network is needed.

\bibliography{Bibliography}

\begin{thebibliography}{10}

\bibitem{algoet1994}
P.~Algoet.
\newblock The strong law of large numbers for sequential decisions under
  uncertainty.
\newblock {\em IEEE Transactions on Information Theory}, 40(3):609--633, 1994.

\bibitem{Bartlett1997}
P.~Bartlett.
\newblock For valid generalization the size of the weights is more important
  than the size of the network.
\newblock In {\em Advances in neural information processing systems}, pages
  134--140, 1997.

\bibitem{bartlett2017}
P.~Bartlett, D.~Foster, and M.~Telgarsky.
\newblock Spectrally-normalized margin bounds for neural networks.
\newblock In {\em Advances in Neural Information Processing Systems}, pages
  6240--6249, 2017.

\bibitem{BiauKLG2010}
G.~Biau, K.~Bleakley, L.~Gy{\"o}rfi, and G.~Ottucs{\'a}k.
\newblock Nonparametric sequential prediction of time series.
\newblock {\em Journal of Nonparametric Statistics}, 22(3):297--317, 2010.

\bibitem{BiauP2011}
G.~Biau and B.~Patra.
\newblock Sequential quantile prediction of time series.
\newblock {\em IEEE Transactions on Information Theory}, 57(3):1664--1674,
  2011.

\bibitem{borovykh2019generalisation}
A.~Borovykh, C.~Oosterlee, and S.~Bohte.
\newblock Generalisation in fully-connected neural networks for time series
  forecasting.
\newblock {\em arXiv preprint arXiv:1902.05312}, 2019.

\bibitem{Breiman1992}
Leo Breiman.
\newblock Probability, volume 7 of classics in applied mathematics.
\newblock {\em Society for Industrial and Applied Mathematics (SIAM),
  Philadelphia, PA}, 1992.

\bibitem{CesaL2006}
N.~Cesa-Bianchi and G.~Lugosi.
\newblock {\em Prediction, Learning, and Games}.
\newblock Cambridge University Press, 2006.

\bibitem{chen2018short}
K.~Chen, K.~Chen, Q.~Wang, Z.~He, J.~Hu, and J.~He.
\newblock Short-term load forecasting with deep residual networks.
\newblock {\em IEEE Transactions on Smart Grid}, 2018.

\bibitem{cranko2018lipschitz}
Z.~Cranko, S.~Kornblith, Z.~Shi, and R.~Nock.
\newblock Lipschitz networks and distributional robustness.
\newblock {\em arXiv preprint arXiv:1809.01129}, 2018.

\bibitem{DevroyeGG2013}
L.~Devroye, L.~Gy{\"o}rfi, and G.~Lugosi.
\newblock {\em A probabilistic theory of pattern recognition}, volume~31.
\newblock Springer Science \& Business Media, 2013.

\bibitem{doukhan1991mixing}
P.~Doukhan.
\newblock {\em Mixing: Properties and Exampless}.
\newblock Departement de mathematique, 1991.

\bibitem{finlay2018improved}
C.~Finlay, A.~Oberman, and B.~Abbasi.
\newblock Improved robustness to adversarial examples using lipschitz
  regularization of the loss.
\newblock {\em arXiv preprint arXiv:1810.00953}, 2018.

\bibitem{gouk2018regularisation}
H.~Gouk, E.~Frank, B.~Pfahringer, and M.~Cree.
\newblock Regularisation of neural networks by enforcing lipschitz continuity.
\newblock {\em arXiv preprint arXiv:1804.04368}, 2018.

\bibitem{GyorfiL2005}
L.~Gy{\"o}rfi and G.~Lugosi.
\newblock Strategies for sequential prediction of stationary time series.
\newblock In {\em Modeling uncertainty}, pages 225--248. Springer, 2005.

\bibitem{GyorfiLU2006}
L.~Gy{\"o}rfi, G.~Lugosi, and F.~Udina.
\newblock Nonparametric kernel-based sequential investment strategies.
\newblock {\em Mathematical Finance}, 16(2):337--357, 2006.

\bibitem{GyorfiS2003}
L.~Gy{\"o}rfi and D.~Sch{\"a}fer.
\newblock Nonparametric prediction.
\newblock {\em Advances in learning theory: methods, models and applications},
  339:354, 2003.

\bibitem{GyorfiUW2008}
L.~Gy{\"o}rfi, F.~Udina, and H.~Walk.
\newblock Nonparametric nearest neighbor based empirical portfolio selection
  strategies.
\newblock {\em Statistics \& Decisions, International Mathematical Journal for
  Stochastic Methods and Models}, 26(2):145--157, 2008.

\bibitem{Gyorfi2007}
L.~Gy{\"o}rfi, A.~Urb{\'a}n, and I.~Vajda.
\newblock Kernel-based semi-log-optimal empirical portfolio selection
  strategies.
\newblock {\em International Journal of Theoretical and Applied Finance},
  10(03):505--516, 2007.

\bibitem{han2015learning}
S.~Han, J.~Pool, J.~Tran, and W.~Dally.
\newblock Learning both weights and connections for efficient neural network.
\newblock In {\em Advances in neural information processing systems}, pages
  1135--1143, 2015.

\bibitem{hornik1989}
K.~Hornik, M.~Stinchcombe, and H.~White.
\newblock Multilayer feedforward networks are universal approximators.
\newblock {\em Neural networks}, 2(5):359--366, 1989.

\bibitem{jiang2017deep}
Z.~Jiang, D.~Xu, and J.~Liang.
\newblock A deep reinforcement learning framework for the financial portfolio
  management problem.
\newblock {\em arXiv preprint arXiv:1706.10059}, 2017.

\bibitem{katz2017reluplex}
G.~Katz, C.~Barrett, D.~Dill, K.~Julian, and M.~Kochenderfer.
\newblock Reluplex: An efficient smt solver for verifying deep neural networks.
\newblock In {\em International Conference on Computer Aided Verification},
  pages 97--117. Springer, 2017.

\bibitem{keskar2016large}
N.~Keskar, D.~Mudigere, J.~Nocedal, M.~Smelyanskiy, and P.~Tang.
\newblock On large-batch training for deep learning: Generalization gap and
  sharp minima.
\newblock {\em arXiv preprint arXiv:1609.04836}, 2016.

\bibitem{kuo2018electricity}
P.~Kuo and C.~Huang.
\newblock An electricity price forecasting model by hybrid structured deep
  neural networks.
\newblock {\em Sustainability}, 10(4):1280, 2018.

\bibitem{LiHG2011}
B.~Li, S.C.H Hoi, and V.~Gopalkrishnan.
\newblock Corn: Correlation-driven nonparametric learning approach for
  portfolio selection.
\newblock {\em ACM Transactions on Intelligent Systems and Technology (TIST)},
  2(3):21, 2011.

\bibitem{liao2018surprising}
Q.~Liao, B.~Miranda, A.~Banburski, J.~Hidary, and T.~Poggio.
\newblock A surprising linear relationship predicts test performance in deep
  networks.
\newblock {\em arXiv preprint arXiv:1807.09659}, 2018.

\bibitem{Meir2000}
R.~Meir.
\newblock Nonparametric time series prediction through adaptive model
  selection.
\newblock {\em Machine learning}, 39(1):5--34, 2000.

\bibitem{miyato2018spectral}
T.~Miyato, T.~Kataoka, M.~Koyama, and Y.~Yoshida.
\newblock Spectral normalization for generative adversarial networks.
\newblock {\em arXiv preprint arXiv:1802.05957}, 2018.

\bibitem{Modha1998}
D.~Modha and E.~Masry.
\newblock Memory-universal prediction of stationary random processes.
\newblock {\em IEEE transactions on information theory}, 44(1):117--133, 1998.

\bibitem{nadaraya1964estimating}
A.~Nadaraya.
\newblock On estimating regression.
\newblock {\em Theory of Probability \& Its Applications}, 9(1):141--142, 1964.

\bibitem{Bousquet2002}
Olivier O.~Bousquet and A.~Elisseeff.
\newblock Stability and generalization.
\newblock {\em Journal of machine learning research}, 2(Mar):499--526, 2002.

\bibitem{Oberman2018lipschitz}
A.~Oberman and J.~Calder.
\newblock Lipschitz regularized deep neural networks generalize.
\newblock 2018.

\bibitem{pollard2012convergence}
D.~Pollard.
\newblock {\em Convergence of stochastic processes}.
\newblock Springer Science \& Business Media, 2012.

\bibitem{rosenblatt1956remarks}
M.~Rosenblatt.
\newblock Remarks on some nonparametric estimates of a density function.
\newblock {\em The Annals of Mathematical Statistics}, pages 832--837, 1956.

\bibitem{mujeeb2019deep}
Sana S.~Mujeeb, N.~Javaid, M.~Ilahi, Z.~Wadud, F.~Ishmanov, and M.~Afzal.
\newblock Deep long short-term memory: A new price and load forecasting scheme
  for big data in smart cities.
\newblock {\em Sustainability}, 11(4):987, 2019.

\bibitem{Sahoo2018}
D.~Sahoo, Q.~Pham, J.~Lu, and S.~Hoi.
\newblock Online deep learning: learning deep neural networks on the fly.
\newblock In {\em Proceedings of the 27th International Joint Conference on
  Artificial Intelligence}, pages 2660--2666. AAAI Press, 2018.

\bibitem{srivastava2014dropout}
N.~Srivastava, G.~Hinton, A.~Krizhevsky, I.~Sutskever, and R.~Salakhutdinov.
\newblock Dropout: a simple way to prevent neural networks from overfitting.
\newblock {\em The Journal of Machine Learning Research}, 15(1):1929--1958,
  2014.

\bibitem{Stout1974}
W.~Stout.
\newblock Almost sure convergence, vol. 24 of probability and mathematical
  statistics, 1974.

\bibitem{szegedy2013intriguing}
C.~Szegedy, W.~Zaremba, I.~Sutskever, J.~Bruna, D.~Erhan, and
  I.~Goodfellowand~R. Fergus.
\newblock Intriguing properties of neural networks.
\newblock {\em arXiv preprint arXiv:1312.6199}, 2013.

\bibitem{talagrand2006generic}
M.~Talagrand.
\newblock {\em The generic chaining: upper and lower bounds of stochastic
  processes}.
\newblock Springer Science \& Business Media, 2006.

\bibitem{Uziel2018}
G.~Uziel and R.~El-Yaniv.
\newblock Growth-optimal portfolio selection under cvar constraints.
\newblock {\em arXiv preprint arXiv:1705.09800}, 2017.

\bibitem{wang2018lstm}
Y.~Wang, Y.~Liu, M.~Wang, and R.~Liu.
\newblock Lstm model optimization on stock price forecasting.
\newblock In {\em 2018 17th International Symposium on Distributed Computing
  and Applications for Business Engineering and Science (DCABES)}, pages
  173--177. IEEE, 2018.

\bibitem{xu2012}
H.~Xu and S.~Mannor.
\newblock Robustness and generalization.
\newblock {\em Machine learning}, 86(3):391--423, 2012.

\bibitem{yu1994rates}
B.~Yu.
\newblock Rates of convergence for empirical processes of stationary mixing
  sequences.
\newblock {\em The Annals of Probability}, pages 94--116, 1994.

\bibitem{zeiler2013stochastic}
M.~Zeiler and R.~Fergus.
\newblock Stochastic pooling for regularization of deep convolutional neural
  networks.
\newblock {\em arXiv preprint arXiv:1301.3557}, 2013.

\bibitem{zhang2016understanding}
C.~Zhang, S.~Bengio, M.~Hardt, B.~Recht, and O.~Vinyals.
\newblock Understanding deep learning requires rethinking generalization.
\newblock {\em arXiv preprint arXiv:1611.03530}, 2016.

\bibitem{zhu2017deep}
L.~Zhu and N.~Laptev.
\newblock Deep and confident prediction for time series at uber.
\newblock In {\em 2017 IEEE International Conference on Data Mining Workshops
  (ICDMW)}, pages 103--110. IEEE, 2017.

\end{thebibliography}
\bibliographystyle{plain}
\end{document}